\documentclass{article} 
\usepackage[preprint]{colm2026_conference}

\usepackage{microtype}
\usepackage{hyperref}
\usepackage{url}
\usepackage{booktabs}

\usepackage{pifont}
\usepackage{graphicx}
\usepackage{subcaption}
\usepackage{wrapfig}
\usepackage{multirow}
\usepackage[most]{tcolorbox}
\usepackage[dvipsnames]{xcolor}

\usepackage{amsmath}
\usepackage{amssymb}
\usepackage{mathtools}
\usepackage{amsthm}
\usepackage{enumitem}

\usepackage{multirow}

\usepackage[capitalize,noabbrev]{cleveref}

\theoremstyle{plain}
\newtheorem{theorem}{Theorem}[section]
\newtheorem{proposition}[theorem]{Proposition}

\newtheorem{corollary}[theorem]{Corollary}
\theoremstyle{definition}

\theoremstyle{remark}
\newtheorem{remark}[theorem]{Remark}

\usepackage{tikz}
\usetikzlibrary{shapes,arrows,positioning,fit,backgrounds}

\usepackage[textsize=tiny]{todonotes}

\usepackage{lineno}

\definecolor{darkblue}{rgb}{0, 0, 0.5}
\hypersetup{colorlinks=true, citecolor=darkblue, linkcolor=darkblue, urlcolor=darkblue}

\definecolor{caseorange}{RGB}{199,124,58}
\definecolor{caselight}{RGB}{252,245,238}
\definecolor{answergreen}{RGB}{60,150,90}

\tcbset{
  casestudy/.style={
    enhanced,
    breakable,
    colframe=caseorange,
    colback=white,
    boxrule=0.8pt,
    arc=2pt,
    left=6pt,
    right=6pt,
    top=6pt,
    bottom=6pt,
    fonttitle=\bfseries,
    coltitle=white,
    colbacktitle=caseorange,
    attach boxed title to top left={
      yshift=-2mm,
      xshift=3mm
    },
    boxed title style={
      boxrule=0pt,
      arc=2pt,
      left=6pt,
      right=6pt,
      top=2pt,
      bottom=2pt
    }
  }
}

\title{Learning to Communicate: Toward End-to-End Optimization of Multi-Agent Language Systems}

\author{
 \textbf{Ye Yu\textsuperscript{1}}, 
 \textbf{Heming Liu\textsuperscript{1}},
 \textbf{Haibo Jin\textsuperscript{1}}, 
 \textbf{Xiaopeng Yuan\textsuperscript{1}}, 
 \textbf{Peng Kuang\textsuperscript{1}}, 
 \textbf{Haohan Wang\textsuperscript{1}\thanks{Corresponding author: haohanw@illinois.edu}}
\\
 \textsuperscript{1} University of Illinois Urbana-Champaign, IL, USA 
}

\setlist[itemize]{noitemsep,leftmargin=*,topsep=0.3ex}
\setlist[enumerate]{noitemsep,leftmargin=*,topsep=0.3ex}

\begin{document}

\ifcolmsubmission
\linenumbers
\fi

\maketitle

\begin{abstract}
Multi-agent systems built on large language models have shown strong performance on complex reasoning tasks, yet most work focuses on agent roles and orchestration while treating inter-agent communication as a fixed interface. Latent communication through internal representations such as key–value caches offers a promising alternative to text-based protocols, but existing approaches do not jointly optimize communication with multi-agent reasoning. Therefore we propose DiffMAS, a training framework that treats latent communication as a learnable component of multi-agent systems. DiffMAS performs parameter-efficient supervised training over multi-agent latent trajectories, enabling agents to jointly learn how information should be encoded and interpreted across interactions. Experiments on mathematical reasoning, scientific QA, code generation, and commonsense benchmarks show that DiffMAS consistently improves reasoning accuracy and decoding stability over single-agent inference, text-based multi-agent systems, and prior latent communication 
methods, achieving +26.7\% on AIME24, 
+20.2\% on GPQA-Diamond, and consistent gains across reasoning benchmarks.



\end{abstract}

\section{Introduction}

Multi-agent systems (MAS) built on large language models have emerged as a powerful paradigm for solving complex reasoning tasks, including mathematics, programming, and commonsense problem solving~\citep{zhang2025multiagentarchitecturesearchagentic, zhang2025aflowautomatingagenticworkflow}. Their effectiveness is commonly attributed to two complementary factors: \textit{(i) system structure and agent role assignment}~\citep{chen2024autoagentsframeworkautomaticagent}, which determine how agents are organized and specialized, and \textit{(ii) inter-agent communication}~\citep{chen2025optimaoptimizingeffectivenessefficiency, shen2025understandinginformationpropagationeffects}, which governs how intermediate reasoning states and information are exchanged. By decomposing complex tasks~\citep{wu2023autogenenablingnextgenllm, hong2024metagptmetaprogrammingmultiagent, qian2024chatdevcommunicativeagentssoftware, li2023camelcommunicativeagentsmind}, multi-agent systems consistently outperform single-agent baselines.

As multi-agent language systems continue to mature, recent research has increasingly explored how to systematically improve their performance through optimization~\citep{Wang_2024, zhang2025multiagentarchitecturesearchagentic, zhang2025aflowautomatingagenticworkflow, chen2024autoagentsframeworkautomaticagent}. Existing efforts have primarily focused on improving agentic capabilities through fine-tuning or prompting, as well as refining workflow design through better role assignment, planning strategies, and execution structures~\citep{shinn2023reflexionlanguageagentsverbal, yao2023reactsynergizingreasoningacting,zhang2025aflowautomatingagenticworkflow, chen2024autoagentsframeworkautomaticagent}. These advances have led to more capable agents and more effective system structures. 

\begin{figure*}[t]
    \centering
    \includegraphics[width=\linewidth]{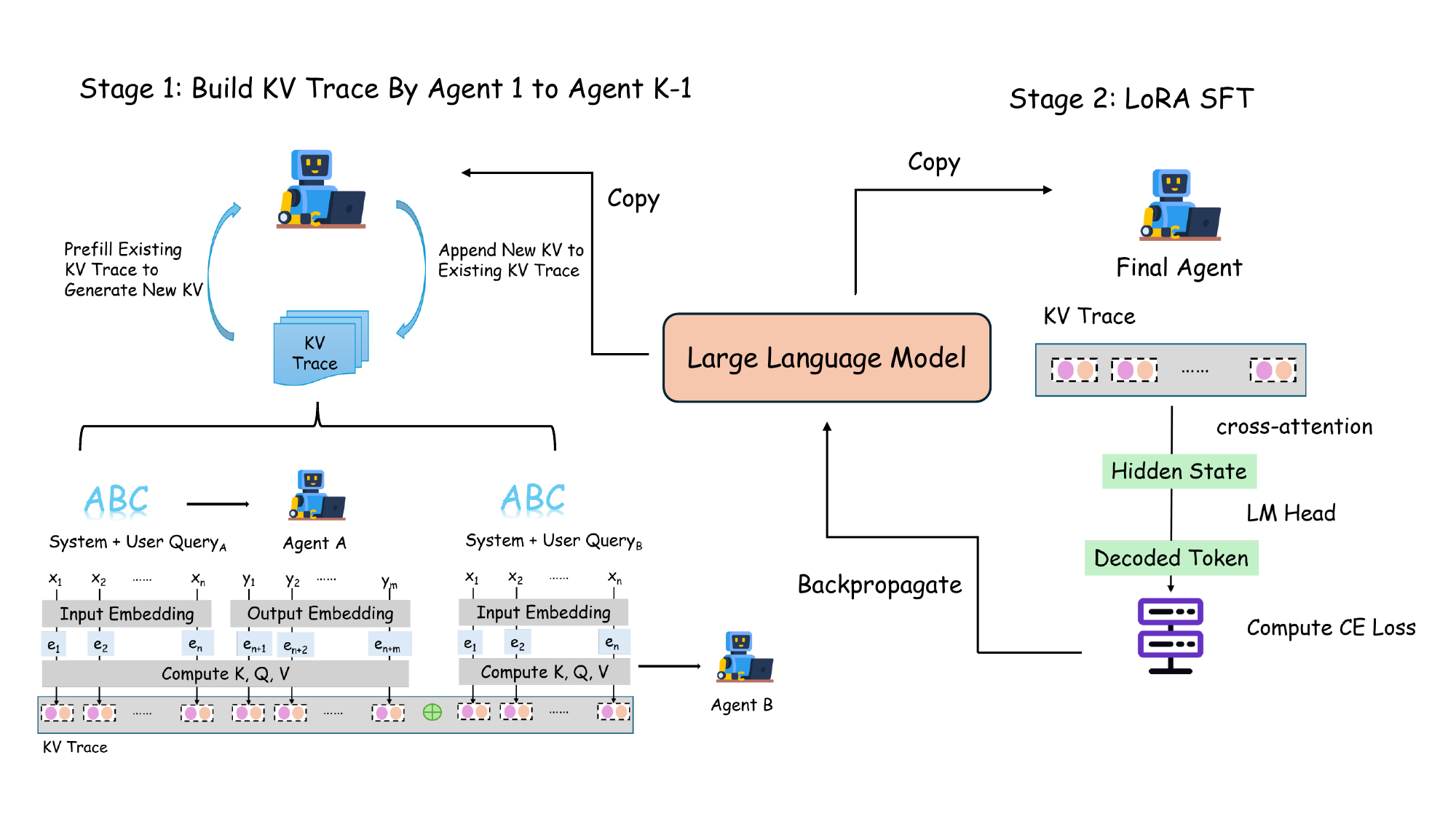}
\vspace{-25pt}
\captionsetup{font=footnotesize}
\caption{In Stage I, agents 1 to K--1 sequentially construct a shared KV trace by prefilling the existing cache and appending newly generated KV segments without gradient updates. The accumulated KV trace serves as a latent communication medium across agents. In Stage II, the final agent performs autoregressive decoding on the prefilled KV cache. Cross-attention over the KV trace produces hidden states, which are projected through the LM head to generate tokens. Supervised fine-tuning is applied using cross-entropy loss, and gradients are backpropagated to update only the LoRA parameters of the final agent while keeping the backbone model frozen.
}
\label{fig:method_overview}
\end{figure*}

However, such improvements largely treat multi-agent systems as modular compositions of optimized components, rather than systems whose interactions can themselves be learned~\citep{zhuge2024languageagentsoptimizablegraphs,shen2025understandinginformationpropagationeffects, chen2025optimaoptimizingeffectivenessefficiency}. A natural next step is therefore to explore whether multi-agent systems can be optimized more holistically, allowing not only agent behaviors and workflows, but also their interactions to improve through learning~\citep{motwani2025maltimprovingreasoningmultiagent, ma2025coevolvingyoufinetuningllm}. Achieving this vision requires moving beyond fixed system interfaces toward interaction mechanisms that can be optimized together with agentic reasoning. As different tasks can place different demands on how intermediate reasoning states should be represented and preserved, effective communication strategies need to adapt to the task rather than relying on fixed protocols~\citep{chen2025optimaoptimizingeffectivenessefficiency,shen2025understandinginformationpropagationeffects}. In this context, enabling communication to become a learnable component represents a critical step towards more fully optimizable multi-agent reasoning systems.

A key obstacle to making communication learnable lies in the discrete nature of how agents currently exchange information. In most existing multi-agent systems, intermediate reasoning results must be decoded into natural language before being transmitted to downstream agents~\citep{chen2023agentverse, hong2024metagptmetaprogrammingmultiagent, chen2023agentverse}. While this design aligns naturally with the token-based interfaces of language models, it introduces a fundamental limitation: discrete message passing creates optimization boundaries between agents~\citep{sukhbaatar2016learningmultiagentcommunicationbackpropagation, foerster2016learningcommunicatedeepmultiagent, mordatch2018emergencegroundedcompositionallanguage, havrylov2017emergencelanguagemultiagentgames}. This limitation is particularly restrictive as the performance of the system depends critically on how well intermediate reasoning results are preserved and interpreted across agents~\citep{deng2023implicitchainthoughtreasoning, wei2023chainofthoughtpromptingelicitsreasoning}. More fundamentally, this reflects a mismatch between the continuous internal representations used during reasoning and the discrete messages used for inter-agent exchange~\citep{zelikman2024quietstarlanguagemodelsteach, hao2025traininglargelanguagemodels, fu2025cachetocachedirectsemanticcommunication, zou2025latentcollaborationmultiagentsystems}. To make communication truly learnable and enable more end-to-end optimization of multi-agent systems, inter-agent information exchange must therefore operate over continuous representations that can support differentiable information flow.


Key--Value (KV) caches provide a natural mechanism for enabling such differentiable communication~\citep{vaswani2023attentionneed}. As continuous internal representations produced by transformer attention layers, KV states can be directly reused by downstream agents without requiring discrete decoding or re-encoding~\citep{kwon2023efficientmemorymanagementlarge}. This allows intermediate reasoning representations to be transmitted through a continuous channel that remains compatible with gradient-based learning. By treating KV caches as a latent communication interface, gradients can propagate not only within individual agents but also across agent boundaries, allowing communication behaviors to be optimized together with the underlying model~\citep{motwani2025maltimprovingreasoningmultiagent, zhuge2024languageagentsoptimizablegraphs, ma2025coevolvingyoufinetuningllm}.

Therefore, we propose DiffMAS, a framework that treats KV-mediated latent communication as a trainable interface between agents and enables SFT-like training over multi-agent interaction trajectories. As illustrated in Figure~\ref{fig:method_overview}, DiffMAS operates in two stages. In Stage I, upstream agents sequentially construct a shared KV trace by prefilling existing KV caches and appending newly generated KV segments, forming a continuous latent communication channel. In Stage II, the final agent performs autoregressive decoding conditioned on the accumulated KV trace, and supervised training is applied to optimize communication jointly with downstream reasoning. By training over full multi-agent interaction trajectories, the framework enables the model to learn communication with agentic reasoning. Our contributions are summarized as follows:

\begin{itemize}
\item We formulate inter-agent communication as an optimization problem in multi-agent language systems, showing that treating communication as a fixed interface prevents end-to-end improvement of how information is transmitted across agents.
\item We propose DiffMAS, a supervised training framework that treats KV-mediated latent communication as a differentiable and task-adaptive interface between agents.
\item We demonstrate through extensive experiments that learning latent communication leads to more stable reasoning trajectories and improves performance, yielding up to +26.7\% improvement on AIME24 and consistent gains across reasoning benchmarks.
\end{itemize}
\section{Related Works}
\subsection{Multi-Agent LLM Systems}
LLM-based multi-agent systems improve complex reasoning by decomposing problems across specialized agents and structured interaction workflows. Prior work has explored improving agent capabilities and system design through role engineering, automated workflow construction, and architecture search \citep{wu2023autogenenablingnextgenllm, li2023camelcommunicativeagentsmind, qian2024chatdevcommunicativeagentssoftware, hong2024metagptmetaprogrammingmultiagent, chen2024autoagentsframeworkautomaticagent, zhang2025aflowautomatingagenticworkflow, zhang2025multiagentarchitecturesearchagentic}. Other studies focus on improving coordination through execution policies or self-configuration\citep{chen2025optimaoptimizingeffectivenessefficiency, wang2025mas, motwani2025maltimprovingreasoningmultiagent, ma2025coevolvingyoufinetuningllm}.

Despite these advances, most multi-agent systems rely on natural language as a universal communication interface~\citep{zhang2025aflowautomatingagenticworkflow, chen2023agentverse}. While flexible, text-based communication requires internal reasoning states to be serialized into discrete tokens\citep{wei2023chainofthoughtpromptingelicitsreasoning}, which can limit information fidelity and prevent communication from being optimized jointly with agent reasoning~\citep{zelikman2024quietstarlanguagemodelsteach, zhuge2024languageagentsoptimizablegraphs, ma2025coevolvingyoufinetuningllm}. These limitations raises the question whether richer internal representations could provide better communication mechanisms.

\subsection{Latent Reasoning}
In parallel, a growing body of work investigates latent or implicit reasoning in large language models, aiming to reduce reliance on explicit textual rationales. Implicit Chain-of-Thought distillation \citep{deng2023implicitchainthoughtreasoning} shows that reasoning behaviors can be internalized into hidden representations, while approaches such as Quiet-STaR \citep{zelikman2024quietstarlanguagemodelsteach} and continuous latent reasoning frameworks \citep{hao2025traininglargelanguagemodels, zhu2025surveylatentreasoning} demonstrate that models can reason effectively without emitting explicit thought tokens.

More recently, these ideas have been extended to multi-agent settings. Several works explore sharing internal representations—such as hidden states or KV caches—between agents to enable richer information exchange than text-based protocols \citep{zou2025latentcollaborationmultiagentsystems, fu2025cachetocachedirectsemanticcommunication, du2026enablingagentscommunicateentirely, fu2026cachetocachedirectsemanticcommunication}. Related studies analyze thought-level communication and aggregation across agents \citep{zheng2025thoughtcommunicationmultiagentcollaboration, feinashley2025mixturethoughtslearningaggregate}. These approaches demonstrate that latent communication can improve efficiency and expressivity by preserving fine-grained intermediate reasoning signals. However, existing latent communication methods either rely on training-free latent exchange\citep{zou2025latentcollaborationmultiagentsystems} or learn representation alignment across models \citep{fu2025cachetocachedirectsemanticcommunication}. In contrast, we treat \textbf{communication itself as a learnable component and optimize it jointly with multi-agent reasoning.}

\section{Method}
\label{sec:method}

\subsection{Problem Setting and Operator View}
\label{sec:method:setting}


We propose DiffMAS, a framework for learning inter-agent communication in multi-agent language systems.We consider a system composed of $K$ sequential agents that collaborate through intermediate reasoning steps. Unlike conventional MAS that treats communication as a fixed protocol, DiffMAS treats communication itself as part of the learnable computation.

Intuitively, this design provides several advantages. First, continuous latent communication allows gradients to propagate across agents, enabling communication strategies to be optimized end-to-end. Second, preserving intermediate representations avoids premature compression into discrete tokens, allowing richer reasoning signals to be transmitted. Third, by constructing communication as part of the computation graph, DiffMAS enables the system to learn how information should be encoded and interpreted across agents.

We now formalize this framework and study its structural properties.

\textbf{From MAS to latent trace representation.}
Each agent produces intermediate latent representations consumed by later agents. Concretely, these are segments of key--value (KV) states generated during autoregressive decoding. We model communication as a sequence of latent blocks, where each stage appends a fixed number of blocks to a shared \emph{latent trace}. This trace is the sole communication medium and is preserved without overwriting.

\textbf{Latent block and trace spaces.}
Let $\mathcal{Z}\subseteq\mathbb{R}^d$ denote the latent block space, where each block $z_t \in \mathcal{Z}$ is a $d$-dimensional vector.
Each stage emits $T$ latent blocks. After $j$ stages (i.e., after $j$ agents have executed sequentially), the total number of blocks is $N_j \triangleq jT$.
Thus, the stage-$j$ trace space is $\mathcal{T}_j = \mathcal{Z}^{N_j}$, and $\mathcal{T}_0 \triangleq \{\mathbf{Z}_{1:0}\}$.
We write the accumulated latent trace as $\mathbf{Z}_{1:N_j} = (z_1,\dots,z_{N_j}) \in \mathcal{T}_j$.


\textbf{Operator view.}
We formalize the multi-agent system as a composition of $K$ differentiable stage operators, each mapping the accumulated latent trace to an extended trace by incorporating new latent information. This abstraction allows us to treat the entire multi-agent reasoning process as a single differentiable computation.

\textbf{Stage operators.}
For each stage $j \in \{1,\dots,K\}$, let $A_\theta^{(j)}(\cdot; x, p_j): \mathcal{T}_{j-1} \to \mathcal{T}_j$ denote a differentiable stage operator, where $x \in \mathcal{X}$ is the input and $p_j$ is the stage-specific prompt, parameterized by $\theta$, such that 
\begin{equation}
\mathbf{Z}_{1:N_j} \;=\; A_\theta^{(j)}\!\big(\mathbf{Z}_{1:N_{j-1}};\, x,p_j\big).
\label{eq:stage-operator}
\end{equation}

The full trajectory is $\mathbf{Z}_{1:N_K}=(A_\theta^{(K)}\circ \cdots \circ A_\theta^{(1)})(\mathbf{Z}_{1:0})$.

\subsection{Within-Stage Micro-Dynamics and Trace Emission}
\label{sec:method:micro}

Each stage $j$ executes $T$ micro-steps of differentiable computation.
Let $s_t^{(j)}\in\mathbb{R}^m$ denote the internal state at micro-step $t$ within stage $j$.

\textbf{Initialization.}
The internal state is initialized as $s_0^{(j)}=\eta_\theta(x,p_j,\mathbf{Z}_{1:N_{j-1}})$.

\textbf{Emission and update order.}
For each micro-step $t\in\{1,\dots,T\}$, the stage first emits a latent block
\begin{equation}
z_t^{(j)}=
g_\theta\!\Big(s_{t-1}^{(j)},x,p_j,\mathbf{Z}_{1:N_{j-1}},\mathbf{Z}^{(j)}_{1:(t-1)}\Big),
\label{eq:emit}
\end{equation}

which is appended to the trace. The state then evolves via a residual update:
\begin{equation}
s_t^{(j)} = s_{t-1}^{(j)} + f_\theta\!\Big(s_{t-1}^{(j)},x,p_j,\mathbf{Z}_{1:N_{j-1}},\mathbf{Z}^{(j)}_{1:t}\Big).
\label{eq:update1}
\end{equation}

All maps $\eta_\theta$, $g_\theta$, and $f_\theta$ are differentiable in $\theta$.

\textbf{Relation to the stage operator.}
The operator $A_\theta^{(j)}$ is the macro-map induced by the initialization, emission, and update rules above.

\subsection{Training Objective and End-to-End Differentiability}
\label{sec:method:training}


After $K$ stages, a decoder produces the output distribution
$p_\theta\big(y|x,\{p_j\}_{j=1}^K\big)=\mathrm{Dec}_\theta(x,p_K,\mathbf{Z}_{1:N_K})$. Training minimizes the supervised negative log-likelihood
$\mathcal{L}(\theta)=-\log p_\theta\!\big(y^\star \mid x,\{p_j\}_{j=1}^K\big)$. Since  $\mathbf{Z}_{1:N_K}$ is obtained via compositions of differentiable operators, gradients of $\mathcal{L}$ propagate through all stages and micro-steps with respect to $\theta$.

\textbf{Practical instantiation.}
In our implementation, $\eta_\theta$, $g_\theta$, and $f_\theta$ are instantiated with a shared pretrained transformer conditioned on stage-specific prompts. Parameter-efficient adaptation (e.g., LoRA) updates only a subset of parameters while keeping the backbone frozen, adapting how latent representations are encoded and interpreted across stages.

\subsection{Structural Implications of Non-Overwriting Communication}
\label{sec:method:implications}

We compare DiffMAS with overwriting communication at the interface level. In overwriting systems, computation is transmitted through a fixed-dimensional carrier $h_j\in\mathbb{R}^{d_h}$ with update $h_{j+1}=F_\theta(h_j,x,p_{j+1})$, so downstream stages access prior computation only through repeated re-encoding. DiffMAS instead communicates through an explicit latent trace $\mathbf{Z}_{1:N_j}\in\mathcal{T}_j$ by concatenation:
\begin{equation}
\mathbf{Z}_{1:N_j}=[\mathbf{Z}_{1:N_{j-1}};\,\mathbf{Z}^{(j)}_{1:T}], 
\qquad \mathbf{Z}^{(j)}_{1:T}=[z^{(j)}_1,\dots,z^{(j)}_T].
\label{eq:append}
\end{equation}
Hence intermediate segments remain explicit in the final trace, while the overwriting carrier has fixed dimension $d_h$ and the ambient dimension of $\mathbf{Z}_{1:N_j}$ grows as $N_j d$.


\begin{proposition}[Interface-induced gradient structure]
\label{prop:gradient}
Consider two $K$-stage systems.
\textbf{(Overwriting.)} Let $h_{j+1}=F_\theta(h_j,x,p_{j+1})$ with $h_j\in\mathbb{R}^{d_h}$. 
Assume $\|\partial h_{j+1}/\partial h_j\|_2 \leq \rho$ for some $\rho\in[0,1)$ and all $j$. 
Then for any $j < K$,
\begin{equation}
    \left\|\frac{\partial \mathcal{L}}{\partial h_j}\right\|_2 
    \;\leq\; \rho^{K-j} \left\|\frac{\partial \mathcal{L}}{\partial h_K}\right\|_2.
    \label{eq:overwrite-decay}
\end{equation}
\textbf{(Concatenation.)} Let $\mathbf{Z}_{1:N_K}=[\mathbf{Z}^{(1)}_{1:T};\dots;\mathbf{Z}^{(K)}_{1:T}]$. 
Then for any $j\in\{1,\dots,K\}$,
\begin{equation}
    \left\|\frac{\partial \mathcal{L}}{\partial \mathbf{Z}^{(j)}_{1:T}}\right\|_2 
    \;\leq\; \left\|\frac{\partial \mathcal{L}}{\partial \mathbf{Z}_{1:N_K}}\right\|_2,
    \label{eq:concat-nodecay}
\end{equation}
introducing no depth-dependent multiplicative factor.
\end{proposition}

\textbf{Implications.}
Proposition~\ref{prop:gradient} shows that DiffMAS avoids the depth-dependent gradient 
attenuation that affects overwriting communication: under the contractive Jacobian 
assumption, the gradient signal reaching stage $j$ in an overwriting system decays 
geometrically as $\rho^{K-j}$, whereas in DiffMAS the communication interface introduces 
no such multiplicative factor for any stage $j$. This means that all intermediate agents 
contribute comparably accessible gradient signal during training, regardless of their 
position in the pipeline. Note that this is an interface-level guarantee: additional 
attenuation may arise from the decoder's attention weights, but no depth-dependent 
factor is introduced by the concatenative interface itself. 
The full proof is in Appendix~\ref{app:structural}.

\textbf{Trade-offs.}
Non-overwriting traces expose more information but enlarge the conditioning object with depth, potentially introducing redundancy or interference. These are structural observations; empirical effects depend on parameterization and training.
\section{Experiements}

\subsection{Experimental Setup}

\begin{table*}[t]
\centering
\footnotesize
\captionsetup{font=footnotesize}
\caption{Accuracy (ACC, \%) comparison across benchmarks. We report absolute
accuracy and improvement over the single-agent baseline. Best results are in bold.}
\resizebox{\linewidth}{!}{%
\begin{tabular}{c c | cc | c | cc}
\toprule
Model & Task 
& \multicolumn{2}{c|}{\shortstack{Text\\Communication}} 
& \shortstack{Training-free\\Latent Communication} 
& \multicolumn{2}{c}{\shortstack{Trained\\Latent Communication}} \\

\cmidrule(lr){3-4} \cmidrule(lr){5-5} \cmidrule(lr){6-7}

& 
& Single & TextMAS 
& LatentMAS 
& C2C & DiffMAS \\

\midrule

\multirow{6}{*}{Qwen3-4B}
& AIME24 & 43.3\% & 46.7\% (+3.4\%) & 50.0\% (+6.7\%) & 0.0\% (-43.3\%) & \textbf{63.3\% (+20.0\%)} \\
& AIME25 & 43.3\% & 43.3\% (+0.0\%) & 50.0\% (+6.7\%) & 0.0\% (-43.3\%) & \textbf{50.0\% (+6.7\%)} \\
& GPQA-Diamond & 42.4\% & 44.9\% (+2.5\%) & 36.4\% (-6.0\%) & 29.8\% (-12.6\%) & \textbf{46.4\% (+4.0\%)} \\
& HumanEval+ & 74.5\% & 76.3\% (+1.8\%) & 78.9\% (+4.4\%) & 33.3\% (-41.2\%) & \textbf{79.8\% (+5.3\%)} \\
& MBPP+ & 63.5\% & 69.3\% (+5.8\%) & 70.6\% (+7.1\%) & 43.9\% (-19.6\%) & \textbf{71.9\% (+8.4\%)} \\
& OpenBookQA & 80.1\% & 81.8\% (+1.7\%) & 77.6\% (-2.5\%) & 79.4\% (-0.7\%) & \textbf{83.2\% (+3.1\%)} \\

\midrule

\multirow{6}{*}{Qwen3-8B}
& AIME24 & 50.0\% & 50.0\% (+0.0\%) & 56.7\% (+6.7\%) & 0.0\% (-50.0\%) & \textbf{76.7\% (+26.7\%)} \\
& AIME25 & 46.7\% & 53.3\% (+6.6\%) & 53.3\% (+6.6\%) & 0.0\% (-46.7\%) & \textbf{56.7\% (+10.0\%)} \\
& GPQA-Diamond & 39.9\% & 43.4\% (+3.5\%) & 45.5\% (+5.6\%) & 41.4\%(+1.5\%) & \textbf{60.1\% (+20.2\%)} \\
& HumanEval+ & 74.5\% & 80.7\% (+6.2\%) & 80.7\% (+6.2\%) & 58.8\%(-15.7\%) & \textbf{81.5\% (+7.0\%)} \\
& MBPP+ & 64.8\% & 69.5\% (+4.7\%) & 74.6\% (+9.8\%) & 57.6\%(-7.2\%) & \textbf{74.8\% (+10.0\%)} \\
& OpenBookQA & 83.6\% & 84.2\% (+0.6\%) & 85.8\% (+2.2\%) & 85.8\%(+2.2\%) & \textbf{85.8\% (+2.2\%)} \\

\bottomrule
\end{tabular}%
}
\label{tab:small_models}
\end{table*}

\textbf{Tasks and Datasets.}
We conduct a comprehensive evaluation of DiffMAS across a diverse set of benchmarks spanning both reasoning-intensive and program synthesis tasks.
Specifically, we evaluate on:
(i) \emph{Math and Science Reasoning}, including AIME 2024~\citep{maxwelljia_aime2024}, AIME 2025~\citep{mathai_aime2025}, and GPQA-Diamond~\citep{rein2024gpqa}, which require multi-step symbolic reasoning and precise numerical or categorical answers;
(ii) \emph{Commonsense Reasoning}, using OpenBookQA~\citep{mihaylov2018suitarmorconductelectricity}, which tests structured reasoning over elementary scientific knowledge; and
(iii) \emph{Code Generation}, including HumanEval-Plus~\citep{liu2023your} and MBPP-Plus~\citep{liu2023your}, which assess functional correctness and generalization in Python program synthesis.

\textbf{Models and Baselines.}
We evaluate DiffMAS across multiple open-source large language models spanning diverse architectures and parameter scales. Our evaluation includes models from the Qwen3 family~\citep{yang2025qwen3} (4B, 8B, and 14B parameters), Ministral-3-8B~\citep{liu2026ministral3}, and DeepSeek-R1-Distill-Qwen-32B~\citep{Guo_2025}, allowing us to assess performance across both mid-scale and large-scale reasoning models.

For baseline comparison, we evaluate DiffMAS against three representative inference paradigms:
(i) \emph{Single-model inference} (Single), where a single language model directly performs standard autoregressive generation;
(ii) \emph{Text-based multi-agent systems} (TextMAS), where multiple agents collaborate through explicit natural language communication; and
(iii) \emph{Training-free latent multi-agent systems} (LatentMAS), where agents communicate through continuous latent representations without additional training; and
(iv) \emph{Cache-to-Cache communication} (C2C), where agents exchange KV-cache representations through learned fusion modules. Note that C2C was originally trained on OpenHermes-2.5~\citep{OpenHermes}, which mainly contains instruction-following data. Detailed inference setup is in Appendix~\ref{app:inference_setup}.

\begin{table*}[t]
\centering
\footnotesize
\captionsetup{font=footnotesize}
\caption{Accuracy (\%) comparison on other large models. We report absolute
accuracy and improvement over the single-agent baseline. Best results are in bold.}
\resizebox{\linewidth}{!}{%
\begin{tabular}{c c | c c c c}
\toprule
Model & Task & Single & TextMAS & LatentMAS & DiffMAS \\
\midrule

\multirow{6}{*}{Mistral3-8B}
& AIME24 & 50.0\% & 26.7\% (-23.3\%) & 33.3\% (-16.7\%) & \textbf{73.3\% (+23.3\%)} \\
& AIME25 & 23.3\% & 23.3\% (+0.0\%) & 33.3\% (+10.0\%) & \textbf{40.0\% (+16.7\%)} \\
& GPQA-Diamond & 47.9\% & 51.0\% (+3.1\%) & 46.4\% (-1.5\%) & \textbf{52.0\% (+4.1\%)} \\
& HumanEval+ & 75.4\% & 76.3\% (+0.9\%) & 74.5\% (-0.9\%) & \textbf{77.2\% (+1.8\%)} \\
& MBPP+ & 69.8\% & 65.3\% (-4.5\%) & 67.4\% (-2.4\%) & \textbf{71.1\% (+1.3\%)} \\
& OpenBookQA & 81.0\% & 80.4\% (-0.6\%) & 82.2\% (+1.2\%) & \textbf{82.2\% (+1.2\%)} \\

\midrule

\multirow{6}{*}{Qwen3-14B}
& AIME24 & 63.3\% & 63.3\% (+0.0\%) & 66.7\% (+3.4\%) & \textbf{76.7\% (+13.4\%)} \\
& AIME25 & 56.7\% & 60.0\% (+3.3\%) & 63.3\% (+6.6\%) & \textbf{66.7\% (+10.0\%)} \\
& GPQA-Diamond & 48.5\% & 51.5\% (+3.0\%) & 52.0\% (+3.5\%) & \textbf{53.0\% (+4.5\%)} \\
& HumanEval+ & 77.2\% & 81.5\% (+4.3\%) & 86.8\% (+9.6\%) & \textbf{87.7\% (+10.5\%)} \\
& MBPP+ & 68.5\% & 72.8\% (+4.3\%) & 75.7\% (+7.2\%) & \textbf{77.2\% (+8.7\%)} \\
& OpenBookQA & 86.6\% & 86.2\% (-0.4\%) & 88.8\% (+2.2\%) & \textbf{89.0\% (+2.4\%)} \\

\midrule

\multirow{6}{*}{\shortstack{DeepSeek-R1-Distill\\Qwen-32B}}
& AIME24 & 66.7\% & 50.0\% (-16.7\%) & 50.0\% (-16.7\%) & \textbf{70.0\% (+3.3\%)} \\
& AIME25 & 50.0\% & 50.0\% (+0.0\%) & 43.3\% (-6.7\%) & \textbf{50.0\% (+0.0\%)} \\
& GPQA-Diamond & 53.3\% & 56.5\% (+3.2\%) & 55.1\% (+1.8\%) & \textbf{57.5\% (+4.2\%)} \\
& HumanEval+ & 80.7\% & 82.4\% (+1.7\%) & 83.3\% (+2.6\%) & \textbf{88.5\% (+7.8\%)} \\
& MBPP+ & 73.8\% & 74.6\% (+0.8\%) & 75.7\% (+1.9\%) & \textbf{78.5\% (+4.7\%)} \\
& OpenBookQA & 91.2\% & 91.2\% (+0.0\%) & 89.6\% (-1.6\%) & \textbf{91.2\% (+0.0\%)} \\

\bottomrule
\end{tabular}%
}
\label{tab:large_models}
\end{table*}

\textbf{Training Setup.}
We perform SFT of DiffMAS on high-performance GPU hardware using parameter-efficient LoRA adaptation. Experiments with Qwen-4B, Qwen-8B, and  Ministral3-8B are conducted on NVIDIA A40 GPUs, while Qwen-14B and DeepSeek-R1-Distill-Qwen-32B are trained on NVIDIA H200 GPUs to accommodate its higher memory and computational requirements. Task-specific LoRA adapters are trained on small subsets of benchmark datasets to specialize the model communication for different reasoning domains. For mathematical and scientific reasoning, we fine-tune on 210 samples from Hendrycks Math~\citep{hendrycksmath2021} for one epoch. For code generation, we train on 50 samples from HumanEval~\citep{chen2021evaluatinglargelanguagemodels} for 10 epochs. For commonsense reasoning, we use 700 samples from CommonsenseQA~\citep{talmor2019commonsenseqaquestionansweringchallenge} for one epoch. We intentionally restrict training small datasets to evaluate whether DiffMAS can learn effective latent communication with minimal supervision. Details of training setup are in Appendix~\ref{app:training_setup}.

\subsection{Main Results}
\label{sec:results}

\begin{wrapfigure}{r}{0.5\linewidth}
    \centering
    \includegraphics[width=\linewidth]{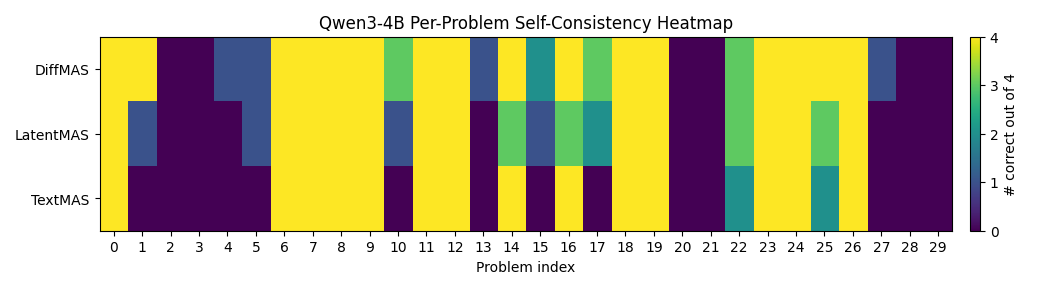}
    \vspace{4pt}
    \includegraphics[width=\linewidth]{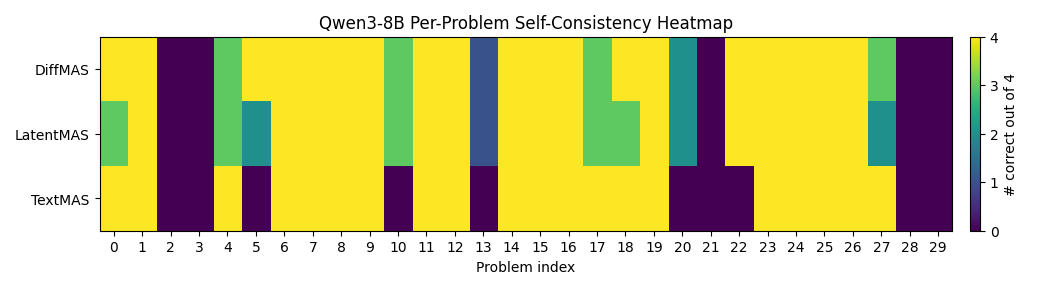}
    \captionsetup{font=footnotesize}
        \vspace{-20pt}
    \caption{Self-consistency analysis on AIME2024 for Qwen3-4B (top) and Qwen3-8B (bottom).}
    \label{fig:self-consistency}
\end{wrapfigure}

\textbf{Overall Performance.}
As shown in Tables~\ref{tab:small_models} and~\ref{tab:large_models}, DiffMAS consistently achieves the best performance across math/science reasoning (AIME24/25, GPQA-Diamond), code generation (HumanEval+, MBPP+), and commonsense reasoning (OpenBookQA). The improvements are especially pronounced at smaller scales, where Qwen3-4B improves from 43.3\% to 63.3\% on AIME24 and reaches the best GPQA-Diamond accuracy of 46.4\%, while Qwen3-8B achieves 76.7\% on AIME24 and 60.1\% on GPQA-Diamond. DiffMAS maintains strong gains at larger scales, including 76.7\% on AIME24 with Qwen3-14B and 70.0\% with DeepSeek-R1-Distill-Qwen-32B. Similar trends hold for code generation, where DiffMAS achieves state-of-the-art results such as 87.7\%/77.2\% (HumanEval+/MBPP+) on Qwen3-14B and 88.5\%/78.5\% at the 32B scale. On commonsense reasoning, DiffMAS consistently achieves the best or tied-best performance, improving OpenBookQA accuracy to 83.2\% on Qwen3-4B and 89.0\% on Qwen3-14B, while matching the best large-model result of 91.2\%. These results demonstrate that DiffMAS delivers consistent improvements across tasks and model scales.



C2C shows relatively weaker performance on difficult reasoning tasks, due to training on OpenHermes-2.5~\citep{OpenHermes}, which focuses on instruction-following data rather than long-horizon reasoning trajectories. This distribution mismatch may limit its effectiveness on math and coding tasks requiring precise multi-step reasoning.
\section{Analysis}
\label{sec:analysis}

\paragraph{Decoding Stability Analysis via Perplexity.}
\label{sec:ppl}
We analyze decoding stability using token-level perplexity (PPL), computed over generated output tokens conditioned on the final agent prompt. Given prompt $x$ and output $y=(y_1,\ldots,y_T)$, we define the average negative log-likelihood $\mathcal{L}(y|x)=-\frac{1}{T}\sum_{t=1}^{T}\log p_\theta(y_t|x,y_{<t})$ and $\mathrm{PPL}(y|x)=\exp(\mathcal{L}(y|x))$. Lower PPL indicates better calibration and more stable decoding. As latent communication modifies the internal attention context, PPL serves as a sensitive measure of how coherently injected KV states are interpreted.

We compare DiffMAS and LatentMAS on AIME~2024 using Qwen3-4B (Figure~\ref{fig:qwen4b_aime2024_ppl}). DiffMAS achieves lower mean perplexity (1.24 vs.\ 1.31) and a tighter distribution, with most values between 1.15--1.30, while LatentMAS exhibits a heavier tail with outliers exceeding 2.0, indicating unstable decoding caused by misaligned latent signals. This pattern also appears at the instance level: while both methods behave similarly on easier problems, LatentMAS shows frequent spikes on harder instances, whereas DiffMAS maintains consistently low perplexity, suggesting that supervised training leads to more structured encoding and interpretation of latent KV representations.

\begin{wrapfigure}{r}{0.4\linewidth}
    \centering
    \includegraphics[width=\linewidth]{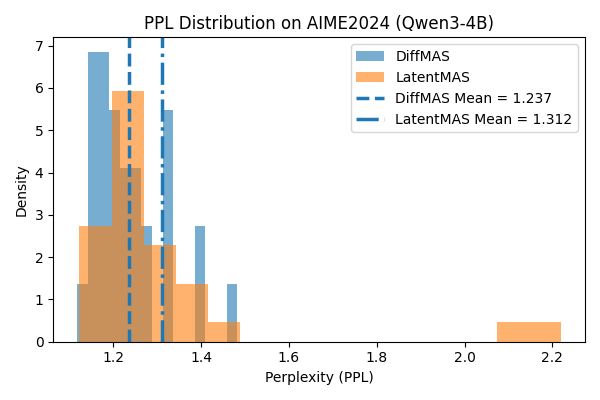}
    \captionsetup{font=footnotesize}
    \caption{Perplexity analysis on AIME2024 for Qwen3-4B, 
    DiffMAS compared to LatentMAS. Density indicates the number of problems falling into each perplexity score category.}
    \label{fig:qwen4b_aime2024_ppl}
\end{wrapfigure}

\paragraph{Self-Consistency in Inference.}
Beyond aggregate accuracy, we analyze the self-consistency of multi-agent reasoning on the AIME 2024 benchmark.
We measure self-consistency by independently sampling each problem four times and recording the number of correct solutions per instance (ranging from 0 to 4).
This metric captures the stability and reliability of the reasoning process under stochastic decoding, which is particularly critical for high-difficulty mathematical tasks.

Figure \ref{fig:self-consistency} shows that DiffMAS consistently achieves higher self-consistency than LatentMAS and TextMAS across both Qwen3-4B and Qwen3-8B. DiffMAS displays a clear shift toward high-consistency outcomes (3–4 correct samples) across most problems, indicating more stable reasoning trajectories. In contrast, LatentMAS frequently produces low-consistency results (0–1 correct samples), suggesting brittle reasoning where small latent communication perturbations lead to divergent outcomes. TextMAS exhibits highly polarized behavior, with many consistent failures (0/4) and only occasional consistent successes (4/4). Overall, these results suggest that DiffMAS improves performance primarily through more stable multi-agent coordination rather than sporadic correct generations, which is particularly important for AIME-style problems where small reasoning errors can invalidate final answers.
\begin{wrapfigure}{r}{0.6\linewidth}
    \centering
    \includegraphics[width=\linewidth]{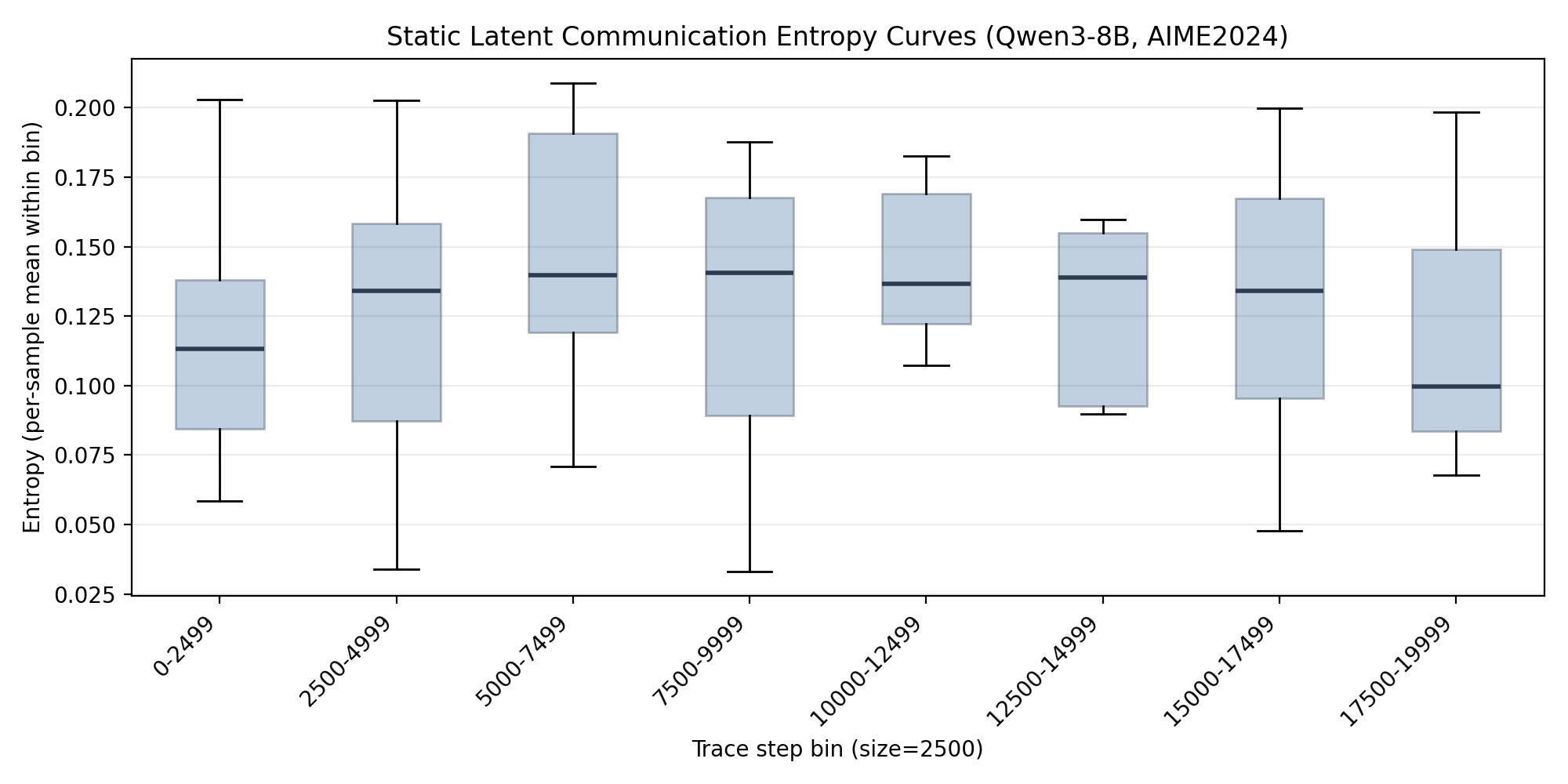}
    \includegraphics[width=\linewidth]{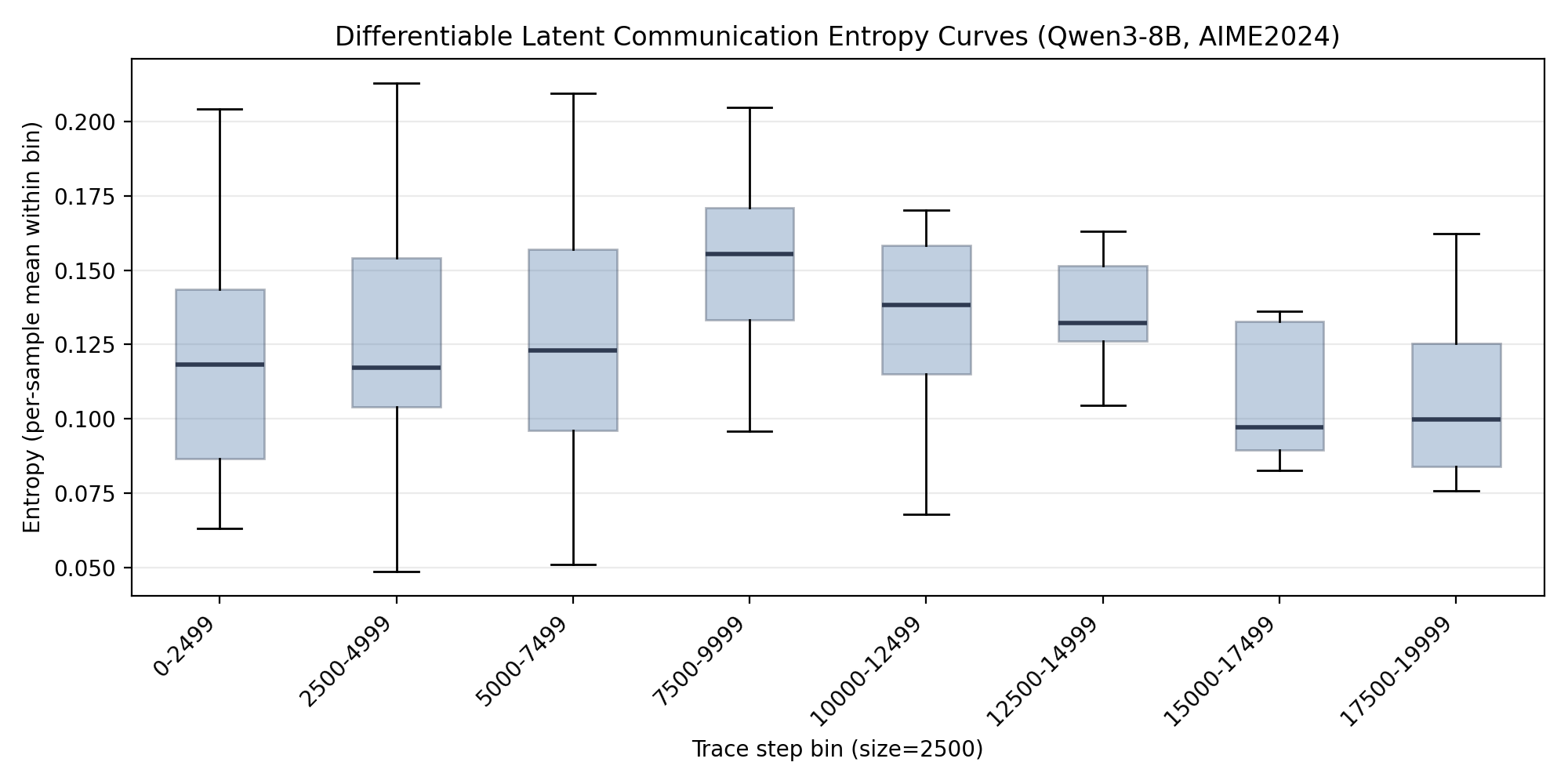}
    \captionsetup{font=footnotesize}
    \caption{
Token-level predictive entropy (top-25) of the judger agent on AIME2024 under LatentMAS (top) and DiffMAS (bottom).
    }
    \label{fig:entropy_judger}
\vspace{-10pt}
\end{wrapfigure}



\paragraph{Token-Level Entropy Dynamics and Stability of Differentiable Latent Communication.}
To study the stability of latent multi-agent communication, we analyze the token-level predictive entropy of the final agent during decoding. At each step, we compute the entropy of the top-25 token distribution averaged over 30 runs, which measures decision uncertainty during reasoning.

Figure~\ref{fig:entropy_judger} shows a structural difference between differentiable latent communication and static latent injection. While entropy increases during decoding in both cases, static latent injection exhibits higher variance and frequent entropy spikes, whereas differentiable communication produces smoother entropy growth with lower peaks.

This difference arises from gradient coupling: in DiffMAS, cross-agent latent states remain within a shared computation graph and are jointly optimized, allowing upstream representations to adapt and reduce downstream uncertainty. In contrast, static latent injection lacks such feedback, forcing downstream agents to compensate via nonlinear remapping, resulting in unstable probability shifts and higher entropy spikes.

\section{Ablation Study}


\textbf{Learning to Solve vs. Learning to Communicate.}
To determine whether DiffMAS gains arise from improved task learning or better communication, we compare it with TextMAS + SFT, a baseline that receives supervised fine-tuning with same training dataset and hyperparameters but uses fixed text-based communication. While TextMAS + SFT adapts the backbone through single-agent SFT, DiffMAS performs end-to-end training over latent multi-agent trajectories, jointly optimizing reasoning and KV-mediated communication.

\begin{wraptable}{r}{0.5\linewidth}
\centering
\footnotesize
\captionsetup{font=footnotesize}
\setlength{\tabcolsep}{3pt}
\caption{Ablation on the number of communication steps for DiffMAS on AIME 2024 using Qwen3-8B.}
\label{tab:ablation_steps}
\resizebox{\linewidth}{!}{%
\begin{tabular}{ccccccc}
\toprule
\# Steps & 0 & 10 & 40 & 100 & 150 & 200 \\
\midrule
Accuracy & 50\% & \textbf{76.7\%} & 63.3\% & 73.3\% & 66.7\% & 63.3\% \\
\bottomrule
\end{tabular}}

\end{wraptable}

As shown in table~\ref{tab:ablation_communication}, both methods achieve identical accuracy on AIME~2024 (76.7\%), suggesting task-level learning alone can explain gains when training and evaluation distributions are similar. However, DiffMAS consistently outperforms TextMAS + SFT on more challenging or distribution-shifted benchmarks, improving AIME~2025 from 50.0\% to 56.7\% and GPQA-Diamond from 53.5\% to 60.1\%.
These results suggest that learning communication provides benefits beyond task-specific fine-tuning alone.

\textbf{Effect of Communication Steps.} We study how the number of latent communication steps affects DiffMAS to determine whether gains come from long latent trajectories or from learning an efficient communication protocol. As shown in table~\ref{tab:ablation_steps}, introducing a small number of communication steps yields large improvements, with accuracy rising from 50.0\% (0 steps) to 76.7\% (10 steps), suggesting that limited latent interaction is sufficient for effective coordination. However, increasing steps beyond this point degrades performance (63.3\% at 40 steps) and without recovery, indicating that longer latent traces introduce redundant or noisy information rather than additional useful signal. Overall, these results suggest that DiffMAS learns a compact communication protocol, with performance improvements emerging from a balance between information integration and noise accumulation.

\begin{table}[t]
\centering
\footnotesize
\captionsetup{font=footnotesize}

\begin{minipage}[t]{0.48\linewidth}
\centering
\captionof{table}{Ablation study comparing DiffMAS and TextMAS+SFT.}
\label{tab:ablation_communication}

\begin{tabular}{lcc}
\toprule
Benchmark & DiffMAS & TextMAS+SFT \\
\midrule
AIME2024      & 76.7\% & 76.7\% \\
AIME2025      & 56.7\% & 50.0\% \\
GPQA-Diamond  & 60.1\% & 53.5\% \\
OpenBookQA    & 85.8\% & 81.2\% \\
\bottomrule
\end{tabular}

\end{minipage}
\hfill
\begin{minipage}[t]{0.48\linewidth}
\centering
\captionof{table}{Comparison between DiffMAS and StitchMAS.}
\label{tab:stitchmas}

\begin{tabular}{lcc}
\toprule
Benchmark & DiffMAS & StitchMAS \\
\midrule
GPQA-Diamond & 60.1\% & 48.4\% \\
AIME2024     & 76.7\% & 60.0\% \\
AIME2025     & 56.7\% & 53.3\% \\
\bottomrule
\end{tabular}

\end{minipage}

\end{table}


\textbf{Stitched vs.\ Continuous Latent Communication under SFT.}   To isolate the impact of latent communication design, we compare DiffMAS with an alternative SFT-based latent multi-agent baseline, \textbf{StitchMAS}, where each agent independently generates KV states that are later concatenated for downstream use. This controls for supervision, as both methods use SFT and differ only in how KV states are produced and integrated. As shown in Table~\ref{tab:stitchmas}, StitchMAS consistently underperforms DiffMAS on Qwen3-8B, with large gaps on GPQA-Diamond (60.1\% vs.\ 48.4\%) and AIME2024 (76.7\% vs.\ 60.0\%). These results suggest that SFT alone is insufficient for effective latent collaboration. Unlike StitchMAS, which produces independently generated KV segments lacking global positional coherence, DiffMAS constructs a continuous KV trace shared across agents. This enables SFT to jointly optimize how intermediate information is encoded and consumed, resulting in more effective inter-agent information flow.

\section{Conclusion}

In this paper, we introduce DiffMAS, a supervised training framework that treats latent communication as a learnable operator within multi-agent interactions. By optimizing full latent trajectories, DiffMAS learns communication with agentic reasoning, improving stability over training-free latent communication. Across reasoning, coding, and commonsense benchmarks, DiffMAS consistently improves accuracy and decoding stability over single-agent, text-based, and latent baselines. More broadly, our results highlight latent communication as an important optimization target for multi-agent systems and motivate future work on fully differentiable MAS with end-to-end optimized interactions.
\bibliography{paper}

\bibliographystyle{colm2026_conference}

\appendix
\section{Structural Properties of Communication Interfaces}
\label{app:structural}

\paragraph{Standing assumptions.}
All statements in this appendix are conditional on explicitly stated
structural assumptions (e.g., contractive Jacobians,
interface dimensional scaling).
No universal expressivity or optimization superiority claim is implied;
the results characterize interface-level consequences under
the stated assumptions.

\subsection{Gradient Path Structure}
\label{app:gradient-structure}

We compare gradient propagation under two abstract interface models.
The following result is conditional on an explicit contraction assumption
that is independent of interface dimensionality.
It serves only to illustrate how multiplicative Jacobian chains,
which arise naturally in overwriting communication,
can introduce depth-dependent attenuation.
No claim is made that fixed-carrier systems necessarily satisfy such contraction,
nor that concatenation prevents attenuation arising from other sources.

\begin{corollary}[Gradient path comparison: overwriting vs.\ concatenation]
\label{cor:grad-compare}
Consider two $K$-stage systems.

\textbf{(Overwriting.)} Let $h_{j+1}=F_\theta(h_j,x,p_{j+1})$ with $h_j\in\mathbb{R}^{d_h}$, where
$F_\theta$ is differentiable in $h_j$. Assume there exists $\rho\in[0,1)$ such that for all $j\in\{0,\dots,K-1\}$,
\begin{equation}
\left\|
\frac{\partial h_{j+1}}{\partial h_j}
\right\|_2 \le \rho.
\label{eq:overwriting-rho}
\end{equation}
Let $\mathcal{L}$ be a differentiable scalar loss that depends on $h_K$.

\textbf{(Concatenation.)} Let $\mathbf{Z}_{1:N_K} = [\mathbf{Z}^{(1)}_{1:T};\dots;\mathbf{Z}^{(K)}_{1:T}]$
be formed by block concatenation, and let $\mathcal{L}$ be a differentiable scalar loss that depends on $\mathbf{Z}_{1:N_K}$.

Then:

\begin{enumerate}[label=(\roman*)]
\item (\emph{Geometric decay under contractive overwriting}) For any $j<K$,
\begin{equation}
\left\|
\frac{\partial \mathcal{L}}{\partial h_j}
\right\|_2
\le
\rho^{K-j}
\left\|
\frac{\partial \mathcal{L}}{\partial h_K}
\right\|_2.
\label{eq:cor-overwriting-decay}
\end{equation}

\item (\emph{No interface-induced multiplicative depth decay from concatenation}) For any $j\in\{1,\dots,K\}$,

\begin{equation}
\left\|
\frac{\partial \mathcal{L}}{\partial \mathbf{Z}^{(j)}_{1:T}}
\right\|_2
\le
\left\|
\frac{\partial \mathcal{L}}{\partial \mathbf{Z}_{1:N_K}}
\right\|_2.
\label{eq:cor-concat-bound}
\end{equation}
In particular, \eqref{eq:cor-concat-bound} introduces no factor that scales like $\rho^{K-j}$ with depth.
\end{enumerate}
The bound in \eqref{eq:cor-concat-bound} isolates the contribution of the
communication interface; additional attenuation may arise from the decoder
or within-stage dynamics.
\end{corollary}

\begin{proof}
(i) By repeated application of the chain rule,
\[
\frac{\partial h_K}{\partial h_j}
=
\prod_{i=j+1}^{K}
\frac{\partial h_i}{\partial h_{i-1}}.
\]
Taking operator norms and using submultiplicativity,
\[
\left\|
\frac{\partial h_K}{\partial h_j}
\right\|_2
\le
\prod_{i=j+1}^{K}
\left\|
\frac{\partial h_i}{\partial h_{i-1}}
\right\|_2
\le
\rho^{K-j},
\]
where the last inequality uses \eqref{eq:overwriting-rho}. Since $\mathcal{L}$ depends on $h_K$,
the chain rule gives
\[
\frac{\partial \mathcal{L}}{\partial h_j}
=
\frac{\partial \mathcal{L}}{\partial h_K}
\frac{\partial h_K}{\partial h_j}.
\]

All norms are Euclidean, and Jacobian operator norms are induced $\ell_2\!\to\!\ell_2$ norms under the standard vectorization of block-concatenated variables.
Taking norms and applying submultiplicativity again yields \eqref{eq:cor-overwriting-decay}.

(ii) Define the concatenation map $C$ by
$C(\mathbf{Z}^{(1)}_{1:T},\dots,\mathbf{Z}^{(K)}_{1:T}) = [\mathbf{Z}^{(1)}_{1:T};\dots;\mathbf{Z}^{(K)}_{1:T}]$.
For each $j$, the Jacobian $\partial \mathbf{Z}_{1:N_K}/\partial \mathbf{Z}^{(j)}_{1:T}$ is the block-insertion
operator $P_j$ that places $\mathbf{Z}^{(j)}_{1:T}$ into the $j$-th coordinate block and zeros elsewhere. This operator
is an isometry on its domain, hence $\|P_j\|_2=1$.
By the chain rule,
\[
\frac{\partial \mathcal{L}}{\partial \mathbf{Z}^{(j)}_{1:T}}
=
\frac{\partial \mathcal{L}}{\partial \mathbf{Z}_{1:N_K}}\,P_j,
\]
so
\[
\left\|
\frac{\partial \mathcal{L}}{\partial \mathbf{Z}^{(j)}_{1:T}}
\right\|_2
\le
\left\|
\frac{\partial \mathcal{L}}{\partial\mathbf{Z}_{1:N_K}}
\right\|_2
\|P_j\|_2
=
\left\|
\frac{\partial \mathcal{L}}{\partial\mathbf{Z}_{1:N_K}}
\right\|_2,
\]
which is \eqref{eq:cor-concat-bound}.
\end{proof}

\begin{remark}[Interface-level distinction]
The depth-dependent factor $\rho^{K-j}$ arises
from the multiplicative Jacobian chain in overwriting.
In contrast, concatenation introduces no multiplicative
depth-dependent factor at the communication interface itself.

\end{remark}

\begin{table}[h]
\centering
\caption{Ablation by replacing latent steps of different agents with zeros during training (Qwen3-8B on AIME2024).}
\label{tab:latent_replace}
\begin{tabular}{lccc}
\toprule
 & No planner \& critic & No critic & No refiner \\
\midrule
Replace at latent & 0.7 & 0.7 & 0.7 \\
Replace at decode & 0.7 & 0.7 & 0.7 \\
\bottomrule
\end{tabular}
\end{table}

From the full DiffMAS accuracy of $76.7\%$, replacing the latent steps of any one agent with zeros reduces performance to $70.0\%$ in all three cases. The key empirical pattern is therefore the \emph{similarity of the degradation across stages}, not the absence of degradation. This pattern is consistent with the interface-level analysis in Appendix~\ref{app:gradient-structure}. Under an overwriting interface, especially under the contractive Jacobian assumption in Corollary~\ref{cor:grad-compare}, the contribution of earlier stages would be expected to attenuate with depth, leading to stage-dependent influence. By contrast, under concatenation, each stage remains directly represented in the final latent trace, so ablating any single stage can produce a more comparable loss. Thus, Table~\ref{tab:latent_replace} is consistent with the view that DiffMAS benefits from a non-overwriting communication interface, in which intermediate computations from different agents remain similarly accessible downstream.
\section{MAS Setup}
\label{sec:mas_setup}

To minimize confounding effects from specialized orchestration designs, we adopt a standard and widely used sequential multi-agent pipeline~\citep{zou2025latentcollaborationmultiagentsystems}. Specifically, we use a sequential four-agent structure consisting of a \textbf{Planner}, \textbf{Critic}, \textbf{Refiner}, and \textbf{Solver}.
This canonical setup is intentionally simple and commonly used in prior LLM-based multi-agent reasoning systems, allowing us to focus on the effect of the \emph{communication medium} (text vs.\ latent KV states) rather than any particular MAS topology or role-engineering heuristics. Across all methods, we keep the agent roles, ordering, and role prompts fixed; the only difference is how intermediate information is transmitted between agents.


\paragraph{Planner.}
The Planner receives the input question and produces a concise step-by-step plan without generating the final answer.

\begin{tcolorbox}[colback=white,colframe=blue!40!gray,title=\textbf{Planner}]
You are a Planner Agent. Given an input question, design a clear, step-by-step plan for how to solve the question.

Question: \{question\}

Your outlined plan should be concise with a few bulletpoints for each step. Do not produce the final answer.
Now output your plan to solve the question below:
\end{tcolorbox}

\paragraph{Critic.}
The Critic evaluates the correctness and completeness of the Planner's plan and provides constructive feedback. The plan is provided via the inter-agent communication channel (text for TextMAS; latent KV representation for LatentMAS/DiffMAS).

\begin{tcolorbox}[colback=white,colframe=blue!40!gray,title=\textbf{Critic}]
Question: \{question\}

You are a Critic Agent to evaluate the correctness of the input plan for the given question and provide helpful feedback for improving the plan.
The plan information is provided in latent KV representation format. Review the plan and question and output:
(1) original plan contents
(2) constructive feedback on the original plan.

Format your response as follows:
Original Plan: [Copy the provided Planner Agent's plan here]
Feedback: [Your detailed feedback to improve the plan here]

Now, output your response below:
\end{tcolorbox}

\paragraph{Refiner.}
The Refiner integrates the original plan and Critic feedback to produce a refined plan that is both correct and concise. It receives (i) latent-format information containing the previous plan with feedback, and (ii) the input question in text format.

\begin{tcolorbox}[colback=white,colframe=blue!40!gray,title=\textbf{Refiner}]
Question: \{question\}

You are a Refiner Agent to provide a refined step-by-step plan for solving the given question.
You are provided with:
(1) latent-format information: a previous plan with feedback
(2) text-format information: the input question you need to solve.

Based on the input, write a refined and improved plan to solve the question. Make sure your output plan is correct and concise.

Now, output your refined plan below:
\end{tcolorbox}

\paragraph{Solver.}
The final agent produces the answer to the target question. It conditions on all accumulated information from upstream agents. Since upstream latent information can occasionally contain irrelevant content, we explicitly instruct the agent to ignore unhelpful context and output only the final answer.

\begin{tcolorbox}[colback=white,colframe=blue!40!gray,title=\textbf{Judger}]
Target Question: \{question\}

You are a helpful assistant. You are provided with latent information for reference and a target question to solve.

The latent information might contain irrelevant contents. Ignore it if it is not helpful for solving the target question.

You must reason step-by-step to solve the provided Target Question without outputting other irrelevant information.

Now, reason step by step and output the final answer inside [YOUR FINAL ANSWER].
\end{tcolorbox}

\section{Training Setup}
\label{app:training_setup}

DiffMAS is trained directly within the multi-agent interaction loop, where inter-agent latent communication is treated as part of the model's conditioning context. During training, the backbone language model remains frozen, and only lightweight LoRA adapters are optimized, enabling parameter-efficient adaptation without modifying the pretrained base model.

For each training instance, intermediate agents first perform latent reasoning steps and produce continuous KV representations that encode their internal computation. These latent KV states are accumulated sequentially across agents and passed to downstream agents as a shared latent context, forming the full multi-agent reasoning trajectory. The final output agent conditions on this accumulated latent context together with its task-specific prompt and generates the target output autoregressively. Training is performed with teacher forcing on the target sequence, and the loss is computed as the negative log-likelihood of the target tokens conditioned on the full latent context and final prompt. Prompt tokens and padding positions are masked out from the loss, ensuring that gradients are driven only by the correctness of the final output. In this way, gradients propagate through the latent communication pathway, allowing the model to jointly adapt how upstream agents encode information into KV states and how downstream agents interpret them.

We train task-specific LoRA adapters on small curated subsets of benchmark datasets in order to specialize latent communication for different reasoning domains. For mathematical reasoning, we construct a 210-sample training set from Hendrycks Math by uniformly sampling 30 samples from each of seven subsets: algebra, counting and probability, geometry, intermediate algebra, number theory, prealgebra, and precalculus. This balanced composition ensures coverage across multiple types of symbolic and quantitative reasoning while keeping the training set intentionally small. For code generation, we randomly choose 50 samples from HumanEval. For commonsense reasoning, we use 700 samples from CommonsenseQA and collected synthetic reasoning traces generated by Gemini-3-Flash-Preview~\cite{gemini3flash_vertexai_2025}. Box~\ref{box:Gemini_Prompt}  contains the prompt we used for generating reasoning traces. We filter the generated reasoning traces by retaining only those for which Gemini produces the correct final answer, resulting in a final set of 700 high-quality traces that are used as supervision.

\begin{tcolorbox}
[colback=white,colframe=blue!40!gray,title=\textbf{Gemini Prompt}]
\label{box:Gemini_Prompt}
Target Question: \{question\}

You are a helpful assistant. You must reason step-by-step to solve the provided Target Question without outputting other irrelevant information. 

Your final answer must be selected from A,B,C,D,E. For example \textbackslash boxed\{A\}. Do not add any other contents inside the box. 

Now, reason step by step and output the final answer inside \textbackslash boxed\{YOUR\_FINAL\_ANSWER\}.

\end{tcolorbox}

Optimization is carried out using AdamW with a cosine learning-rate schedule and a linear warmup ratio of $0.03$. LoRA parameters are trained with a fixed learning rate of $5\times10^{-5}$, using rank $r=8$, scaling factor $\alpha=16$, and dropout rate $0.05$, while the backbone model remains frozen. Gradients are accumulated over 64 micro-batches per optimization step, and gradient norms are clipped to value of $1.0$.

\section{Inference Setup}
\label{app:inference_setup}
Across all baselines and our method, we use the same decoding configuration. All LLM agents are run with a temperature of $0.6$ and a top-$p$ value of $0.95$. The maximum output length for the final output agent is adjusted according to its task difficulty: 2{,}048 tokens for OpenBookQA; 4{,}096 tokens for MBPP+ and HumanEval+; 8{,}192 tokens for GPQA-Diamond; and 32{,}768 tokens for AIME~2024/2025. No additional sampling heuristics are introduced.

For evaluation coverage, we run experiments on the full benchmark sets of AIME~2024, AIME~2025, GPQA-Diamond, MBPP+, and HumanEval+, excluding the 50 training samples used for supervised training. For OpenBookQA, we evaluate on a random 1000 samples in train split in main subset.

\section{Case Study}
\label{app:case_study}

\begin{tcolorbox}[
  casestudy,
  title={Case Study (LatentMAS on AIME2024)}
]

\textbf{Input Question:}
Let $\mathcal{B}$ be the set of rectangular boxes with surface area $54$ and volume $23$. Let $r$ be the radius of the smallest sphere that can contain each of the rectangular boxes that are elements of $\mathcal{B}$. The value of $r^2$ can be written as $\frac{p}{q}$, where $p$ and $q$ are relatively prime positive integers. Find $p+q$.

\vspace{0.5em}
\hrule
\vspace{0.5em}

\textbf{LatentMAS Output:}

Okay, let's try to solve this problem step by step. So, we have a tetrahedron ABCD with specific edge lengths. The question is about finding the equal distance from an interior point I to each of the tetrahedron's faces. That distance is given in a particular form, and we need to find m + n, and p, then sum them.

First, I remember that in a tetrahedron, if there's a point equidistant from all four faces, that point is called the incenter. But wait, in a tetrahedron, the incenter exists only if the tetrahedron is isohedral or has some symmetry. But here, the given edge lengths seem symmetric. Let me check the edge lengths again: $AB=CD$, $CD=\sqrt{41}$, $AC=BD=\sqrt{80}$, $BC=AD$. So, the edges are symmetric in pairs. For example, $AB=CD$, $AC=BD$, $BC=AD$. This suggests that the tetrahedron is isosceles or maybe even a special type of tetrahedron, like a disphenoid...
\newline
\newline
\textcolor{blue}{[Attempt Solving Steps Omitted]}
\newline
\newline
Alternatively, let me think about the volume. Let me try to compute the volume using the Cayley-Menger determinant. The formula for the volume of a tetrahedron with edge lengths AB = a, AC = b, AD = c, BC = d, BD = e, CD = f is given by the Cayley-Menger determinant:

V = sqrt( |CM| / 2 2 2 2 2 2 2 2 2 2 2 2 2 2 2 2 2 2 2 2 2 2 2 2 2 2 2 2 2 2 2 2 2 2 2 2 2 2 2 2 2 2 2 2 2 2 2 2 2 2 2 2 2 2 2 2 2 2 2 2 2 2 2 2 2 2 2 2 2 2 2 2 2
\newline
\newline
\textcolor{red}{[Chaotic Decoding Omitted]}
\newline
\newline
query query enjoys eventually enjoys query transformations
\textcolor{red}{(\ding{55} Wrong Final Answer)}

\end{tcolorbox}

\begin{tcolorbox}[
  casestudy,
  title={Case Study (DiffMAS on AIME2024)}
]

\textbf{Input Question:}
Let $\mathcal{B}$ be the set of rectangular boxes with surface area $54$ and volume $23$. Let $r$ be the radius of the smallest sphere that can contain each of the rectangular boxes that are elements of $\mathcal{B}$. The value of $r^2$ can be written as $\frac{p}{q}$, where $p$ and $q$ are relatively prime positive integers. Find $p+q$.

\vspace{0.5em}
\hrule
\vspace{0.5em}

\textbf{DiffMAS Output:}

Okay, let's try to solve this problem step by step. The question is about finding the square of the radius of the smallest sphere that can contain any rectangular box in the set $B$, which has a surface area of $54$ and volume $23$. Then, we need to express this radius squared as a reduced fraction $\frac{p}{q}$ and find $p+q$.

First, I need to recall that the smallest sphere that can contain a rectangular box would have a diameter equal to the space diagonal of the box. The radius would then be half of that diagonal. So, the radius $r$ is $\frac{\sqrt{x^2 + y^2 + z^2}}{2}$, where $x$, $y$, $z$ are the dimensions of the box. Therefore, $r^2$ would be $\frac{x^2 + y^2 + z^2}{4}$. So, our goal is to find the minimum possible value of $x^2 + y^2 + z^2$ given the constraints of surface area and volume.
\newline
\newline
\textcolor{blue}{[Consistent and Stable Solving Steps Omitted]}
\newline
\newline
4. **Compute the Maximum Value**: The configuration where $x = y = 2$ and $z = \frac{23}{4}$ satisfies the constraints and yields the maximum $x^2 + y^2 + z^2 = \frac{657}{16}$. Thus, $r^2 = \frac{657}{64}$.

5. **Simplify the Fraction**: The fraction $\frac{657}{64}$ is already in lowest terms since 657 and 64 are coprime. Therefore, $p = 657$ and $q = 64$, giving $p + q = 721$.

**Final Answer**:  
\boxed{721} \textcolor{answergreen}{(\checkmark\ Correct Final Answer)}

\end{tcolorbox}

To better understand how inter-agent communication mechanisms influence reasoning behavior and decoding stability, we present a qualitative case study comparing three representative paradigms: natural language (TextMAS), raw KV caches without adaptation (LatentMAS), and our learned latent KV communication (DiffMAS). We analyze a representative example from AIME 2024, using identical agent roles, prompts, and decoding configurations across all settings. 
This analysis complements the quantitative results in Section~\ref{sec:results} by examining how communication design affects reasoning trajectories at the instance level.

\paragraph{Text-based Communication.}
TextMAS exhibits the most stable inference behavior among the three paradigms. Since pretrained language models are optimized for next-token prediction over natural language, exchanging intermediate reasoning through text closely matches the model’s training distribution. As a result, decoding tends to remain well-structured and logically organized.

However, this robustness comes at the cost of representational capacity. Intermediate reasoning must be compressed into discrete tokens before being passed to downstream agents. This serialization process inevitably introduces information loss, particularly for subtle intermediate computations or implicit contextual signals. In complex mathematical reasoning problems, we observe that downstream agents sometimes need to reconstruct missing details, which can lead to incomplete reasoning chains or reduced solution accuracy.

\paragraph{Training-free Latent Communication.}
LatentMAS removes this textual bottleneck by allowing agents to directly share KV caches. This enables downstream agents to attend to richer internal representations that preserve nearly all upstream reasoning information. In some cases, this additional expressivity allows agents to capture useful intermediate structure that would otherwise be lost in textual summaries.

However, we observe substantial instability in the reasoning process. Intermediate reasoning steps frequently drift off-topic, introduce irrelevant associations, or inconsistently track variables. These behaviors appear to stem from a distribution mismatch: pretrained transformers are not trained to interpret externally injected KV states as a semantic communication interface. Consequently, latent representations can function as structured but uncalibrated signals that disrupt attention patterns.

In the analyzed example, this manifests as chaotic reasoning trajectories in which correct intermediate ideas appear but are later abandoned or misused. Although some reasoning segments remain locally plausible, the global reasoning process often becomes incoherent, ultimately leading to incorrect final answers.

\paragraph{DiffMAS: Learned Latent Communication.}
DiffMAS addresses this stability–expressivity tradeoff by explicitly learning how agents should communicate in latent space. Through supervised fine-tuning over full multi-agent latent reasoning trajectories, the model jointly learns both how intermediate agents should encode useful information into KV states and how downstream agents should interpret these representations.

In the same AIME example, DiffMAS produces substantially more coherent reasoning behavior. Intermediate steps remain focused on relevant constraints, variable definitions remain consistent across agents, and the reasoning trajectory follows a clear progression toward the solution. Compared to LatentMAS, we observe fewer attention failures and less reasoning drift.

Importantly, the final answer produced by DiffMAS is correct, demonstrating that supervised training aligns latent communication with task objectives rather than leaving it as an implicit and brittle interface. This suggests that the primary advantage of DiffMAS is not merely richer representations, but the ability to structure how those representations are generated and consumed across agents.

\paragraph{Summary of Observations.}
This case study highlights a fundamental tradeoff between stability and expressivity in multi-agent communication design. Text-based communication provides strong stability due to alignment with pretraining distributions but suffers from information compression. Training-free latent communication improves expressivity but introduces instability due to distribution mismatch. DiffMAS resolves this tension by learning a task-aligned latent communication protocol that preserves expressivity while maintaining stable decoding behavior.

These qualitative observations are consistent with our quantitative findings on decoding stability, perplexity, and self-consistency, which show that DiffMAS achieves both improved reasoning accuracy and more reliable multi-agent coordination.

\end{document}